\documentclass{article}

\usepackage{microtype}
\usepackage{graphicx}
\usepackage{subcaption}
\usepackage{booktabs} 
\usepackage{amsthm}

\usepackage{hyperref}
\usepackage{thm-restate}
\usepackage{lipsum,booktabs}

\usepackage[preprint]{icml2026}

\usepackage{amsmath}
\usepackage{amssymb}
\usepackage{mathtools}
\usepackage{amsthm}

\usepackage[capitalize,noabbrev]{cleveref}

\theoremstyle{plain}
\newtheorem{theorem}{Theorem}[section]

\newtheorem{lemma}[theorem]{Lemma}
\newtheorem{corollary}[theorem]{Corollary}
\theoremstyle{definition}
\newtheorem{definition}[theorem]{Definition}

\theoremstyle{remark}

\usepackage[textsize=tiny]{todonotes}

\usepackage{csquotes}

\icmltitlerunning{Caterpillar of Thoughts}

\renewcommand{\epsilon}{\varepsilon}
\renewcommand{\emptyset}{\varnothing}

\newcommand{\card}[1]{|#1|}

\newcommand{\poly}{\textnormal{poly}}

\newcommand{\floor}[1]{\left\lfloor #1 \right\rfloor}

\DeclareMathOperator*{\argmin}{arg\,min}

\newcommand{\opt}{\mathsf{OPT}}
\newcommand{\nbo}{\opt_{\mathcal{NB}}}
\newcommand{\oapx}{\widehat{\opt}}

\newcommand{\dist}{\textnormal{dist}}
\renewcommand{\P}{\mathcal{P}}

\begin{document}

\twocolumn[

  \icmltitle{Caterpillar of Thoughts:\\ The Optimal Test-Time Algorithm for Large Language Models}

  \icmlsetsymbol{equal}{*}

  \begin{icmlauthorlist}
    \icmlauthor{Amir Azarmehr}{equal,yyy}
    \icmlauthor{Soheil Behnezhad}{equal,yyy}
    \icmlauthor{Alma Ghafari}{equal,yyy}
  \end{icmlauthorlist}

  \icmlaffiliation{yyy}{Northeastern University }

 \icmlcorrespondingauthor{ }

  \vskip 0.3in
]
\newcommand{\snote}[1]{{\color{brown} [Soheil: #1]}}
\newcommand{\aanote}[1]{{\color{purple} [Amir: #1]}}
\newcommand{\agnote}[1]{{\color{blue} [Alma: #1]}}


\printAffiliationsAndNotice{\icmlEqualContribution}

\begin{abstract}
 Large language models (LLMs) can often produce substantially better outputs when allowed to use additional test-time computation, such as sampling, chain of thought, backtracking, or revising partial solutions. Despite the growing empirical success of such techniques, there is limited theoretical understanding of how inference time computation should be structured, or what constitutes an optimal use of a fixed computation budget.

We model test-time computation as an algorithm interacting with a Markov chain: at any point, the algorithm may resume generation from any previously observed state. That is, unlike standard Markov chains where the states are drawn ``passively'', we allow the algorithm to ``backtrack'' to any previously observed state of the Markov chain at any time.

Many of the existing test-time algorithms such as Chain-of-Thought (CoT) \cite{wei2023chainofthoughtpromptingelicitsreasoning}, Tree-of-Thoughts (ToT) \cite{yao2023treethoughtsdeliberateproblem}, or Best-of-$k$ \cite{brown2024largelanguagemonkeysscaling} could be seen as specific algorithms in this model.

We prove that while backtracking can reduce the number of generations exponentially, a very limited form of backtracking is theoretically sufficient. Namely, we show that the optimal algorithm always generates a ``caterpillar'' tree. That is, if we remove the leaves of the state tree generated by the optimal algorithm, we obtain a path.
Motivated by our characterization of the optimal algorithm, we present Caterpillar of Thoughts (CaT), a new test-time computation algorithm, reducing the number of token/state generations. Our empirical evaluation shows that CaT, compared to ToT, achieves a better success rate while also reducing the number of token generations.
\end{abstract}

\section{Introduction}

Recent advances in large language models have demonstrated that performance can often be improved not by additional training, but by allocating more computation at inference time. Techniques such as self consistency \cite{wang2023selfconsistencyimproveschainthought}, 
tree-of-thought exploration \cite{yao2023treethoughtsdeliberateproblem}, 
iterative refinement \cite{madaan2023selfrefine}, recursive self-aggregation \cite{venkatraman25selfaggregation},  verifier guided search \cite{lightman2023verify, botta2025querycomplexityverifierassistedlanguage, snell2024scalingllmtesttimecompute, wang2025valueguidedsearchefficientchainofthought}, and backtracking \cite{ vonrütte2025generalizedinterpolatingdiscretediffusion,yang2025stepleapforwardselfbacktracking} allow an LLM to sample multiple continuations, revisit earlier partial solutions, and selectively pursue promising directions. These methods have proven effective across various reasoning, planning, and creative generation tasks, suggesting that inference time computation plays a central role in enabling more capable and adaptive language agents.


Despite this progress, the algorithmic foundations of test-time computation remain poorly understood.  Most existing approaches are motivated heuristically, and it is unclear how close they are to optimal, or how performance should scale with additional inference-time resources. More fundamentally, {\em given a fixed but non-trivial budget, what is the best possible way to deploy test-time compute?}

\subsection{Our Contributions}

\paragraph{A Theoretical Model of Test-Time Computation:} To answer the question above, we need a theoretical model for test-time computation. We model this problem as exploring a (massive) {Markov chain}, which is defined based on a pretrained LLM. 
Each state in the Markov chain represents a \emph{partial solution}, i.e., a sequence of tokens.
From any state, the LLM induces a distribution over successor states via sampling new tokens. 
That is, the Markov chain randomly transitions to a new state, corresponding to the solution obtained by appending the new tokens to the current partial solution.
Crucially, modern inference time methods allow the algorithm to rewind: rather than continuing only from the most recently generated state, it may resume generation from any previously observed state.


This leads naturally to the model of a fully observable Markov chain with rewinding \cite{azarmehr2026markovchainsrewinding}, in which an algorithm adaptively chooses both where to rewind and how long to explore. The goal is to reach a designated target state representing, for example, a correct answer or a sufficiently high-quality output, while minimizing the expected number of steps (i.e., total number of observed states). We remark that allowing the algorithm to rewind can exponentially decrease the number of steps (This is discussed further in \cref{sec:dummy}, see \cref{fig:dummy-chain} for an example).

\begin{figure}
    \centering
    \resizebox{0.5\textwidth}{!}{
    \usetikzlibrary{arrows.meta,calc}

\tikzset{
    green node/.style={circle, draw, fill={rgb,255:red,184;green,233;blue,134}, minimum size=9mm, inner sep=1pt},
    white node/.style={circle, draw, fill=white, minimum size=9mm, inner sep=1pt},
    arrow style/.style={->, >=stealth, line width=0.75pt}
}

\begin{tikzpicture}[
  >=Stealth,
  node distance=1.5cm and 1.5cm,
  state/.style={draw, circle, minimum size=9mm, inner sep=1pt, align=center},
  lab/.style={font=\small, inner sep=1pt}
]

\node[green node, state] (x0) {$x_0$};
\node[green node, state, right=of x0] (x1) {$x_1$};
\node[green node, state, right=of x1] (x2) {$x_2$};
\node[lab, right=of x2] (dots1) {$\cdots$};
\node[green node, state, right=of dots1] (xn) {$x_n$};

\node[state, below=1.5cm of x2] (D) {$D$};

\draw[->] (x0) -- node[lab, pos=0.5, above] {$p$} (x1);
\draw[->] (x1) -- node[lab, pos=0.5, above] {$p$} (x2);
\draw[->] (x2) -- node[lab, pos=0.5, above] {$p$} (dots1);
\draw[->] (dots1) -- node[lab, pos=0.5, above] {$p$} (xn);
\draw[->] (x0) -- node[lab, pos=0.3, left] {$1-p$} (D);
\draw[->] (x1) -- node[lab, pos=0.3, left] {$1-p$} (D);
\draw[->] (x2) -- node[lab, pos=0.3, left] {$1-p$} (D);
\draw[->] (xn) -- node[lab, pos=0.3, left] {$1$} (D);
\draw[->] (D) edge[loop below] node[lab] {$1$} (D);
\node[lab] at ($(dots1)!0.4!(D)$) {$\ldots$};

\end{tikzpicture}
    }
    \caption{A Markov chain where rewinding is crucial, consisting of $n+1$ states $x_0 \rightarrow x_1 \rightarrow  \ldots \rightarrow x_{n}$ on a path, and one dummy state $D$.
    The arrows denote the transition probabilities.
    The first state on the path, $x_0$, is the starting point of the algorithm, and the last state, $x_{n}$, is the target state.
}
    \label{fig:dummy-chain}
\end{figure}
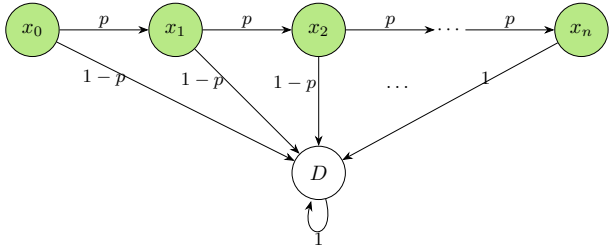

The set of observed states can be interpreted as a tree: each observed state corresponds to a node of the tree, and rewinding to a previous state opens a new branch at the corresponding node.
Many existing test-time inference strategies fit into this framework.
For example, Chain~of~Thoughts (CoT)\cite{wei2023chainofthoughtpromptingelicitsreasoning} explores a single path in the Markov chain, starting at a carefully chosen initial state.
Best of $k$ and self-consistency \cite{wang2023selfconsistencyimproveschainthought} can be interpreted as exploring multiple paths that share the same root, and are otherwise independent.
Other approaches such as Tree of Thoughts (ToT)~\cite{yao2023treethoughtsdeliberateproblem} and Monte Carlo Tree Search (MCTS) \cite{daineseMTCS, park2024ensemblinglargelanguagemodels, LiuMCTS} employ more intricate methods for exploring the tree.
Our rewinding Markov chain model can simulate these strategies by appropriate choices of the rewinding distribution over the observed set.
Moreover, recent verifier-guided backtracking methods can also be expressed within our framework: in \cref{sec:vgb} we show that the proposed Markov chain of \citet{rohatgi2025tamingimperfectprocessverifiers} is recovered as a special case of our model, and our framework generalizes it.
\vspace{-3mm}

\paragraph{The Optimal Algorithm:} Our main contribution is characterizing the optimal solution in this model, i.e., presenting an algorithm for exploring the Markov chain that reaches the target state with the minimum number of steps, in expectation.
We show that no matter what the Markov chain looks like, the optimal generated tree is always a {\em caterpillar}.
That is, all the nodes are within distance $1$ of a central path, and the algorithm may rewind only to withdraw the most recently generated state. 
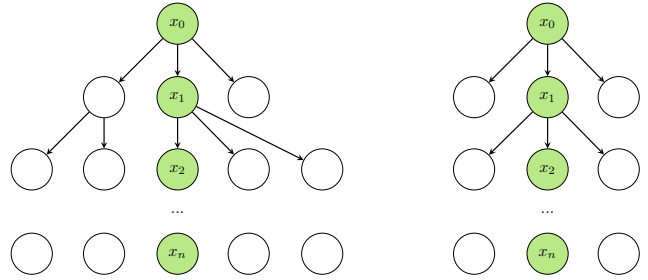
\begin{figure}[h]
    \centering
    \resizebox{0.5\textwidth}{!}{
    \tikzset{
    green node/.style={circle, draw, fill={rgb,255:red,184;green,233;blue,134}, minimum size=9mm, inner sep=1pt},
    white node/.style={circle, draw, fill=white, minimum size=9mm, inner sep=1pt},
    arrow style/.style={->, >=stealth, line width=0.75pt}
}

\begin{tikzpicture}[x=0.75pt,y=0.75pt,yscale=-1,xscale=1]
\def\levelzero{55.77}
\def\levelone{116.77}
\def\leveltwo{176.77}
\def\leveln{246.77}

\node[green node] (x0-left) at (176.9,\levelzero) {$x_0$};

\node[white node] (a1-left) at (115.9,\levelone) {};
\node[green node] (b1-left) at (176.9,\levelone) {$x_1$};
\node[white node] (c1-left) at (235.9,\levelone) {};

\node[white node] (a2-left) at (55.9,\leveltwo) {};
\node[white node] (b2-left) at (115.9,\leveltwo) {};
\node[green node] (c2-left) at (176.9,\leveltwo) {$x_2$};
\node[white node] (d2-left) at (235.9,\leveltwo) {};
\node[white node] (e2-left) at (296.9,\leveltwo) {};

\node at (176.9,212) {$...$};

\node[white node] (a3-left) at (55.9,\leveln) {};
\node[white node] (b3-left) at (115.9,\leveln) {};
\node[green node] (c3-left) at (176.9,\leveln) {$x_n$};
\node[white node] (d3-left) at (235.9,\leveln) {};
\node[white node] (e3-left) at (296.9,\leveln) {};

\draw[arrow style] (x0-left) -- (a1-left);
\draw[arrow style] (x0-left) -- (b1-left);
\draw[arrow style] (x0-left) -- (c1-left);

\draw[arrow style] (b1-left) -- (c2-left);
\draw[arrow style] (b1-left) -- (d2-left);
\draw[arrow style] (b1-left) -- (e2-left);

\draw[arrow style] (a1-left) -- (a2-left);
\draw[arrow style] (a1-left) -- (b2-left);

\node[green node] (x0-right) at (483.9,\levelzero) {$x_0$};

\node[white node] (a1-right) at (422.9,\levelone) {};
\node[green node] (b1-right) at (483.9,\levelone) {$x_1$};
\node[white node] (c1-right) at (542.9,\levelone) {};

\node[white node] (a2-right) at (422.9,\leveltwo) {};
\node[green node] (b2-right) at (483.9,\leveltwo) {$x_2$};
\node[white node] (c2-right) at (542.9,\leveltwo) {};

\node at (483.9,212) {$...$};

\node[white node] (a3-right) at (422.9,\leveln) {};
\node[green node] (b3-right) at (483.9,\leveln) {$x_n$};
\node[white node] (c3-right) at (542.9,\leveln) {};

\draw[arrow style] (x0-right) -- (a1-right);
\draw[arrow style] (x0-right) -- (b1-right);
\draw[arrow style] (x0-right) -- (c1-right);

\draw[arrow style] (b1-right) -- (a2-right);
\draw[arrow style] (b1-right) -- (b2-right);
\draw[arrow style] (b1-right) -- (c2-right);

\end{tikzpicture}
    }
    \caption{This figure illustrates the tree explored by two different algorithms, with \cref{fig:dummy-chain} being the input Markov chain. 
    White circles represent state $D$, and green states represent the optimal path to the target state.
    The right side tree represents the explored tree by CaT, and the left side tree represents a BFS-like algorithm that expands without pruning the dummy state. }
    \label{fig:placeholder}
\end{figure}
More precisely, for each state $x$, we define a value $\opt(x)$, which is the optimal number of steps that an algorithm requires in expectation to reach the target state from $x$.
This value guides our algorithm in exploring the Markov chain, and can be computed exactly given full knowledge of the Markov chain, or otherwise approximated.
That is, the algorithm maintains a \emph{minimizer state} $x^*$, initialized to the starting state.
Then, it repeatedly generates the next state $g$ from $x^*$, until it obtains a state with $\opt(g) < \opt(x^*)$, which in turn replaces $x^*$ and undergoes the same process.
This is formalized below as \cref{alg:opt}.

\begin{algorithm}
    \caption{Caterpillar of Thoughts}
    \label{alg:opt}
    \begin{algorithmic}

    \STATE \textbf{Input:}
    An initial state $x_0$, and a target state $z$
    
    \STATE $x^* \gets x_0$

    \WHILE{$x^* \neq z$}
        \STATE $g \gets $ next state drawn from $x^*$
        \IF{$\opt(g) < \opt(x^*)$}
            \STATE $x^* \gets g$
        \ENDIF
    \ENDWHILE
     \end{algorithmic}
\end{algorithm}

\begin{restatable}{theorem}{nonbranch}\label{thm:optimal-non-branching}
    Given an initial state $x_0$, \cref{alg:opt} reaches the target state using the optimal number of steps in expectation, i.e., in $\opt(x_0)$ steps.
\end{restatable}

Observe that the algorithm always keeps track of only two states, the minimizer state $x^*$, and the most recently generated state $g$.
It may rewind to $x^*$ or otherwise continue from $g$, resulting in a tree that consists of leaves hanging from a central path, i.e., a caterpillar.

\paragraph{Massive Markov Chains:} We note that, in practice, computing the exact value of $\opt(x)$ might not be feasible given that the Markov chain corresponding to a pre-trained LLM could be huge. A nice property of \cref{alg:opt} is that only access to the values of $\opt(\cdot)$ is sufficient for its implementation. In other words, given oracle access to $\opt(\cdot)$, additional knowledge of the Markov chain's structure is not helpful. In this paper, we also study whether approximations of this oracle are sufficient. In particular,  in \cref{sec:noise}, we present an algorithm that reaches the target state in $\poly(\opt(x))$ steps, given only noisy approximations of $\opt$.

\textbf{Experiments:} 
In \cref{sec:experiments}, we evaluate a variant of the optimal \Cref{alg:opt} (which is more robust when we only have approximations of $\opt(x)$) empirically on two standard reasoning benchmarks of the literature: the Game of 24 and $5\times5$ crosswords. See \Cref{tab:game24-results} for a direct comparison between this algorithm and Tree-of-Thoughts (ToT). While our algorithm improves the success rate significantly from 74\% to 81\%, it also reduces the number of generated tokens (and thus the cost of evaluations) by more than 20\%.

\begin{table}[t]
\caption{Performance comparison on the Game of 24 benchmark. Tree-of-Thoughts uses a beam size of $5$. 
The proposed method reports best of $2$ runs, for maximum budget of iterations (steps) set to be $10$ and $15$. By tokens in this table we mean sum of generated tokens and prompt tokens.}
\centering
\begin{tabular}{lcccc}
\toprule
Method & Success Rate(\%) & Avg. Tokens \\
\midrule
Tree-of-Thoughts \\\cite{yao2023treethoughtsdeliberateproblem} & 74  & 
19.2k \\
\midrule
\textbf{CaT} 
\\
(15 steps) &  \textbf{81} & \textbf{
15.3k}  \\
\midrule
\textbf{CaT} 
\\(10 steps) &  \textbf{78} & \textbf{14.2k
}\\
\bottomrule
\end{tabular}

\label{tab:game24-results}
\end{table}

\section{Preliminaries}
\paragraph{The Formal Model:}
We consider a finite-state fully observable Markov chain
$M=(\Omega,P)$ with transition probabilities $P:\Omega\times\Omega\to[0,1]$.
In the rewinding model, unlike a standard Markov chain, in each step the process is allowed to \emph{rewind} to any previously visited state.

Given a fully observable Markov chain $M = (\Omega, P)$, an initial state $x_0 \in \Omega$, and a target state $z \in \Omega$, the goal is to devise a rewinding algorithm that minimizes the expected number of steps to obtain the target state $z$.  

\begin{definition}[Rewinding Algorithms]
    Given a fully observable Markov chain $M= (\Omega,P)$, and  an initial state $x_0$, in every step $t \geq 0$, a rewinding algorithm chooses a rewinding time $t' \leq t$, and obtains a new state $x_{t + 1}$ according to the transition probabilities of $x_{t'}$.
    We refer to $x_{t'}$ as the parent state of $x_{t+1}$, and denote it by $p(x_{t+1})$.
\end{definition}

Our main contribution is characterizing the optimal algorithm, which reaches the target state in the minimum number of steps.

\begin{definition}[The Optimal Algorithm]
    Given a tree of states observed so far, $X = (x_0, \cdots, x_t)$, a rewinding algorithm for the remainder of the process is optimal if it minimizes the expected number of steps before obtaining the target state $z$.
    We use $\opt(X)$ to denote the expected hitting time of such an optimal algorithm.
\end{definition}

For the remainder of the paper, we use $\mathsf{O}(\cdot)$ instead of $\opt(\cdot)$ for brevity.
\let\oldopt\opt
\renewcommand{\opt}{\mathsf{O}}

\section{Fully Observable Markov Chain}

\subsection{Optimal Algorithm}

This section is devoted to the proof of our main theorem, \cref{thm:optimal-non-branching}.
First, we establish that the optimal algorithm depends only on the set of distinct states observed so far, and the exact structure of the tree is inconsequential. The proof is deferred to \cref{sec:proofs}.

\begin{restatable}{lemma}{lemmaOptByS}\label{lem:opt-by-S}
There exists an optimal rewinding algorithm such that
given a sequence of states $X = (x_0, \cdots, x_t)$,
the rewinding time $t' \leq t$ chosen for obtaining the next state is determined solely based on the set of distinct states $S$ in $X$. 
We use $\opt(S) := \opt(X)$ to denote the optimal hitting time corresponding to $S$, and $\opt(x)$ as a shorthand for $\opt(\{x\})$.    
\end{restatable}

Next, we define non-branching rewinding algorithms, which are algorithms that evolve a caterpillar tree.

\begin{definition}[Non-branching Algorithms]
A rewinding algorithm is \textbf{non-branching} if in any step $t \geq 0$,
it either rewinds to the most recently obtained state $x_t$,
or to $p(x_t)$, i.e.\ to the same rewinding time as the previous step.
For a state $a \in \Omega$, we let $\nbo(a)$ denote the minimum expected hitting time of $z$ for non-branching algorithms starting at $a$.
\end{definition}

The key idea behind the proof of \cref{thm:optimal-non-branching} is that the optimal algorithm is a non-branching one.

\begin{theorem}\label{thm:main}
    For any set of states $S \subseteq \Omega$, there exists a state $x \in S$ such that
    $$
    \opt(S) = \nbo(x).
    $$
\end{theorem}
\begin{proof}
    We prove the statement by a downward induction on $\card{S}$. It holds trivially for any set of states $S$ that contains $z$ since we have that $\opt(S) = \nbo(z) = 0$.
    This includes the base case, $S = \Omega$.

    For the induction step, given a set $S$ with $\card{S} < \card{\Omega}$, fix any optimal algorithm, and let $a \in S$ be the rewinding state chosen for obtaining the next state.
    Let $g$ be the state that is obtained as a result.
    Observe that if $g \in S$, the optimal algorithm remains the same (by \cref{lem:opt-by-S}).
    Therefore, without loss of generality, we can assume that the optimal algorithm rewinds to $a$, until it obtains a state $g \notin S$.

    For a state $g \notin S$, let $S_g$ denote $S \cup \{g\}$.
    Since $\card{S_g} > \card{S}$, by the induction hypothesis, there exists a state $x_g \in S_g$ such that 
   \begin{equation}
\label{eq:mc1}
\opt(S_g) = \nbo(x_g).
\end{equation}
    That is, there exists a non-branching algorithm starting at $x_g$ that is optimal for $S_g$. 
    We consider two scenarios.
    First, we consider the case that for all $g \notin S$ it holds that $x_g = g$ (this includes the case where $g = z$, and the process stops), i.e., the target is obtained or there is an optimal algorithm for $S_g$ that starts at $x_g$ and does not branch.
    In this case, we have a non-branching algorithm which is optimal for $S$ that starts at $a$, meaning $\opt(S) = \nbo(a)$.
    This algorithm, rewinds to $a$ until a state $g \notin S$ is obtained, and then proceeds according to the non-branching algorithm at $x_g = g$.

    Now, consider the case where there exists a state $g \notin S$ such that $x_g \neq g$.
    Then, it holds that $x_g \in S$.
    We show that $\opt(S) = \nbo(x_g)$ (the algorithm could have started at $x_g$ instead of $a$ to begin with).
    First, observe that since $S \subseteq S_g$, any algorithm for $S$ is a valid algorithm for $S_g$. Therefore we have 
    \begin{equation}
\label{eq:mc2}
\opt(S_g) \le \opt(S).
\end{equation}
Second, since $x_g \in S$, the optimal non-branching algorithm starting at $x_g$ is a valid algorithm for $S$.
Hence,
\begin{equation}
\label{eq:mc3}
\opt(S) \le \nbo(x_g).
\end{equation}
Finally, combining \cref{eq:mc1,eq:mc2,eq:mc3}, we have
\[
\opt(S) \le \nbo(x_g) = \opt(S_g) \le \opt(S),
\]
which implies that the above inequalities hold with equality, and hence,
\begin{equation*}
    \opt(S) = \nbo(x_g).\qedhere
\end{equation*}
\end{proof}

This directly implies that, with a set of states $S$ observed so far, the optimal algorithm is a non-branching one, which starts at the state $x \in S$ that minimizes $\opt(x)$.

\begin{corollary} \label{cor:minimizer}
Given an initial set of states $S$, we have $$\opt(S) = \min_{x \in S} \opt(x) = \min_{x \in S} \nbo(x). $$
\end{corollary}
\begin{proof}
First, note that for any state $x \in \Omega$, 
\cref{thm:main} implies $\opt(x) = \nbo(x)$ (by letting $S = \{x\}$).
Therefore, $\opt(x)$ and $\nbo(x)$ can be used interchangeably and the second equality in the claim follows.

Now, we show $\opt(S) = \min_{x \in S} \nbo(x)$.
Observe that a non-branching algorithm starting at a state $x \in S$ is a valid algorithm for $S$. Hence, 
$\opt(S) \leq \min_{x \in S} \nbo(x)$.
Moreover, by \cref{thm:main}, there exists a state $x^* \in S$ such that $\opt(S) = \nbo(x^*)$.
Therefore, $\opt(S) = \min_{x \in S} \nbo(x)$.
\end{proof}

We restate \cref{thm:optimal-non-branching} here.

\nonbranch*

\begin{proof}
    Note that the \cref{alg:opt} is non-branching since in every step, it either rewinds to $x^*$, the same rewinding state as the last step, or to $g$, the most recently obtained state.

    To show it is the optimal non-branching algorithm, by \cref{cor:minimizer}, it suffices to prove that whenever the algorithm rewinds to $x^*$, it satisfies $x^* = \argmin_{x \in S} \nbo(x)$, where $S$ denotes the set of states obtained so far.
    This can be proven inductively as follows.
    For the base case, it holds trivially since $x^*$ is equal to $x_0$, the only state in $S$.
    For the induction step, observe that if $\opt(g) \geq \opt(x^*)$, then $x^*$ is still the minimizer of $\nbo(x)$.
    Otherwise, if $\opt(g) < \opt(x^*)$, then $g$ is the new minimizer of $\nbo(x)$ in $S$, and $x^*$ is changed to $S$.
\end{proof}

We use \cref{thm:optimal-non-branching} to obtain a recursive expression of $\opt(\cdot)$, which we then utilize to compute the values of $\opt(\cdot)$, in \cref{sec:computing-opt}.

\begin{restatable}{corollary}{corOptRecursion}
\label{cor:opt-recursion}
    For a state $x \neq z$, let $L \subseteq \Omega$ be the set of states $y$ such that $\opt(y) < \opt(x)$.
    Then, it holds that
    $$
    \opt(x) = \frac{1 + \sum_{y \in L}P(x, y) \opt(y)}{P(x, L)}.
    $$
\end{restatable}
\vspace{-3mm}
\subsection{Computing $\opt(x)$}
\label{sec:computing-opt}

We present a Dijkstra-like algorithm for computing the value of $\opt(x)$ for every state $x$.
The algorithm maintains an upper bound $d_x$ for $\opt(x)$ along with a set of states $T$ that grows in every step.
At any point, for every $x \in T$, it holds that $d_x = \opt(x)$.
The algorithm is formalized as \cref{alg:dijkstra}.

\begin{algorithm}
    \caption{Computes the optimal hitting time for each initial state.}
    \label{alg:dijkstra}
    \begin{algorithmic}

    \STATE \textbf{Input:}
    A Markov chain $M = (\Omega, P)$, and a target state $z \in \Omega$.
    
    \STATE $T \gets \emptyset$
    
    \STATE $d_x \gets \begin{cases}
        0 & \textnormal{if $x = z$, and } \\
        \infty & \textnormal{otherwise}
    \end{cases}$

    \WHILE{$T \neq \Omega$ and $\min_{x \notin T}d_x < \infty$}
        \STATE $x^* \gets \argmin_{x \notin T} d_x$ \hfill (break ties arbitrarily)
        
        \STATE $T \gets T \cup \{x^*\}$
        \FOR{$x \notin T$}
            \IF{$P(x, T) > 0$}
                \STATE $d_x \gets \frac{1 + \sum_{y \in T}P(x, y) d_y}{P(x, T)}$
            \ENDIF
        \ENDFOR
    \ENDWHILE

    \RETURN{$\{d_x\}_{x \in \Omega}$}
     \end{algorithmic}
\end{algorithm}

\begin{restatable}{theorem}{thmDijkstra} \label{thm:Dijkstra}
    \cref{alg:dijkstra} computes $\opt(x)$ for every state $x \in \Omega$ and runs in time $O(n^2)$, where $n = \card{\Omega}$ denotes the number of states in the Markov chain.
    If $m$ entries of the transition probabilities $P(\cdot, \cdot)$ are non-zero, then the algorithm can be implemented in time $O((m + n) \log n)$.
\end{restatable}

\subsection{The Necessity of Rewinding}
\label{sec:dummy}

In this section, we discuss the role of rewinding in the efficiency of algorithms.
We show that allowing the algorithm to rewind to previous states can exponentially improve the expected run time, compared to algorithms that do not allow rewinding or allow limited forms of rewinding.
Essentially, without rewinding, the algorithm may get \enquote{lost} and never recover.

Consider, as an example, the Markov chain depicted in \cref{fig:dummy-chain}, with $p = \frac{1}{2}$.
The optimal algorithm operates as follows.
At any point, it maintains a state $x_i$, which is the furthest state observed in the path (initially, $x_0$).
It repeatedly rewinds to $x_i$ to generate the next state, until $x_{i+1}$ is obtained.
Since each increase takes $\frac{1}{p}$ generations in expectation,
the algorithm reaches $x_n$ in $\frac{n}{p} = 2n$ steps.

To compare, we consider an algorithm that allows no rewinding.
That is, the algorithm starts at $x_0$ and explores in a single trajectory in hopes of reaching $x_n$.
Observe that in each step, the trajectory either moves one state further down the path, or moves to the dummy state $D$.
In case it moves to $D$, it is absorbed forever, and never reaches the target state $x_n$.
Therefore, the algorithm reaches the target state with a negligible probability of $\frac{1}{2^n}$.

Similarly, we can consider an algorithm with limited backtracking that explores a number of independent trajectories in parallel.
Based on the analysis above, each trajectory reaches the target state with a probability of $\frac{1}{2^n}$.
Hence, an algorithm that explores $k$ trajectories in parallel reaches the target state with a probability of 
\begin{equation*}
1 - \left(1 - \frac{1}{2^n}\right)^k \leq \frac{k}{2^n}.
\end{equation*}
Therefore, any algorithm that reaches the target state with a constant probability must explore at least $\Omega(2^n)$ states, which is exponentially larger than the expected number of steps for the optimal rewinding algorithm, $2n$.
\section{Massive Markov Chains}
\label{sec:noise}

\cref{thm:optimal-non-branching} shows that an optimal strategy can be devised given oracle access to the exact values of $\opt(x)$.
However, such oracle access may be infeasible for large Markov chains, since examining the structure of the entire chain is computationally inefficient.
Therefore, in this section, we consider access to noisy approximations of $\opt(x)$ under two different noise models.

In \cref{sec:laplacian-noise}, we consider a model in which the algorithm can obtain independent approximations of $\opt(x)$ with Laplacian noise, and present a non-branching rewinding strategy with a hitting time comparable to $\opt(x)$. 
In \cref{sec:adversarial-noise}, we show that adversarially chosen $(1 \pm \epsilon)$-approximations of $\opt(x)$ are not sufficient for an efficient algorithm, implying that some assumptions about the noise (such as in the first noise model) are necessary.

\subsection{A Polynomial Algorithm under Laplacian Noise} \label{sec:laplacian-noise}

First, we formalize the noise model.

\begin{definition}[Approximations with Laplacian Noise]
\label{def:laplacian-noise}
For each state $s \in\Omega$, the algorithm can repeatedly obtain independent noisy evaluations of $\opt(s)$:
\[
\oapx(s) \sim \mathrm{Laplace}(\opt(s), \lambda),
\]
which has a Laplacian distribution of mean $\opt(s)$ and scale $\lambda$ (i.e., variance $2\lambda^2$), each evaluation incurs the same runtime cost as generating a state.\footnote{The Laplace distribution has a density function of $f(x) = \frac{1}{2\lambda}e^{\card{x-\mu}/\lambda}$. We remark that the algorithm in this section can be easily adapted to any independent noisy evaluation of variance $\lambda$.}
\end{definition}

Of special interest may be the case where $\sqrt{2} \lambda = \epsilon \opt(s)$, as $\oapx(s) \in (1 \pm \epsilon)\opt(s)$ with constant probability.

Before stating the algorithm for this model, henceforth referred to as the stable algorithm, we describe an auxiliary algorithm which is used in the analysis.
The auxiliary algorithm is non-branching and utilizes exact values of $\opt(x)$, in a manner that enables the stable algorithm to simulate it using only noisy approximations of $\opt(x)$.
As opposed to the optimal algorithm, which replaces the minimizer state $x^*$ with any state that has a smaller optimal hitting time, the auxiliary algorithm replaces $x^*$ only when the new state improves the optimal hitting time non-negligibly.

\begin{definition} \label{def:aux}
    Given an initial state $x_0 \in \Omega$ and a parameter $\epsilon > 0$, the auxiliary algorithm operates as follows:
    \begin{enumerate}
        \item Maintain a special state $x^*$ (serving as an \enquote{approximate minimizer}), initially equal to $x_0$.
        \item Repeatedly rewind to $x^*$ to obtain a state $g$, until the new state satisfies $g = z$ or $\opt(g) < \opt(x) - c$, where $c = \frac{\epsilon}{1+\epsilon}$. \label{step:aux-2}
        \item If $g$ is equal to the target state $z$, halt. Otherwise, let $x^* := g$ and go back to step \ref{step:aux-2}.
    \end{enumerate}
    We use $\opt'(x_0)$ to denote the expected hitting time of $z$ with this algorithm.
\end{definition}

We establish that the hitting time of the auxiliary algorithm is at most $(1 + \epsilon)$ times larger than the optimal hitting time. The proof is deferred to \cref{sec:proofs}.

\begin{restatable}{lemma}{lemmaAuxAnalysis}\label{lem:aux-analysis}
    For any state $x$, it holds that
    $$
    \opt'(x) \leq (1 + \epsilon)\opt(x).
    $$
\end{restatable}

Now, we move on to the stable algorithm. 
First, we note that the mean-median technique can be used to strengthen the approximation guarantee of $\oapx(\cdot)$ (as defined in \cref{def:laplacian-noise}).
The proof is deferred to \cref{sec:proofs}.

\begin{restatable}{lemma}{meanMedian}\label{lem:mean-median}
    For any state $s \in \Omega$ and parameters $\epsilon$ and $\delta$, 
    given access to estimations $\oapx(s) \sim \mathrm{Laplace}(\opt(s), \lambda)$,
    the algorithm can use $O(\frac{\lambda^2}{\epsilon^2} \log \frac{1}{\delta})$ estimates to obtain a value $X_s$ such that 
    $
    \card{X_s - \opt(s)} \leq \epsilon,
    $
    with probability $\delta$.
\end{restatable}

With the mean-median subroutine at hand, we are ready to present the stable algorithm (\cref{alg:stable}).

\begin{algorithm}
    \caption{Stable Algorithm (for noisy approximations)}
    \label{alg:stable}
    \begin{algorithmic}

    \STATE \textbf{Input:}
    A Markov chain $M = (\Omega, P)$, a target state $z \in \Omega$, an upper bound $N$ for $\opt(x_0)$, and a Laplacian approximator of $\oapx(\cdot)$ of the hitting times
    
    \STATE $x^* \gets x_0$

    \WHILE{$x^* \neq z$}
        \STATE $g \gets $ next state drawn from $x^*$
        
        \STATE $X_{x^*}, X_g \gets$ mean-median approximations obtained from $\oapx$ (\cref{lem:mean-median}), with parameters $\epsilon = \frac{1}{10}$ and $\delta = \frac{1}{10 N}$
        
        \IF{$X_g < X_{x^*} - \frac{1}{2}$}
            \STATE $x^* \gets g$
        \ENDIF
    \ENDWHILE
    \end{algorithmic}
\end{algorithm}

\begin{restatable}{theorem}{thmStable}\label{thm:stable}
    Given an initial state $x_0$, \cref{alg:stable} reaches the target state with an expected computation cost (number of steps plus number of approximations) of at most $O(\lambda^2 \cdot \opt(x_0) \log \opt(x_0))$.
    In particular, when $\oapx(s) \sim \mathrm{Laplace}\left(\opt(s), \epsilon \opt(s)\right)$,
    the total cost is at most
    $O(\epsilon^2 \cdot \opt(x_0)^3 \log \opt(x_0))$, in expectation.
\end{restatable}
\subsection{A Lower Bound for Adversarial Noise} \label{sec:adversarial-noise}

In this section, we present a lower bound for adversarial noise in partially observable (i.e., massive) Markov chains, showing that access to arbitrary $(1 \pm \epsilon)$-approximations of $\opt(x)$ is not sufficient for reaching the target state efficiently, when the errors are set arbitrarily by an adversary.
That is, we construct a partially observable Markov chain $M = (\Omega, P)$, with an initial state $x_0 \in \Omega$ and a target state $z$. 
For any generated state $x$, the algorithm can only observe $\oapx(x)$, which is an adversarially set value within $(1 \pm \epsilon)\opt(x)$.
We show that to reach the target state with constant probability, the algorithm must generate at least $\exp(\opt(x_0))$ states. 

\begin{restatable}{theorem}{thmLB} \label{thm:lb}
    Given a parameter $\epsilon < 1$, there exists a partially observable Markov chain $M = (\Omega, P)$, with observations $\oapx(x) \in  (1 \pm \epsilon)\opt(x)$ for any state $x$, such that any algorithm that can only access the observations requires $\exp( \opt(x_0) )$ steps to reach the target, where $\opt(x_0)$ can be taken arbitrarily large.
\end{restatable}

Here, we describe the Markov chain and defer the proofs to \cref{sec:proofs}. 
To do so, we use a $\Delta$-regular tree of depth $10 d$.
Here, $\Delta$ is a sufficiently large constant, and $d$ is taken as large as desired to increase $\opt(x_0)$.
The root of the tree, which has degree $\Delta$, corresponds to the initial state $x_0$.
Each descendant of $x_0$ has $\Delta - 1$ children.
This continues up to a depth of $10 d$, where each descendant of depth $10 d$ is a leaf and has no children.
Each vertex in the tree corresponds to a state in the Markov chain.
For the transition probabilities, we let any non-leaf state transition to its neighbors in the tree with uniform probability.
A leaf vertex transitions to its parent with probability $\frac{1}{\Delta}$, and remains in the same state otherwise.
Finally, we choose one of the leaves as the target state $z$.
Before presenting the observations $\oapx(x)$, we calculate the optimal hitting times $\opt(x)$ for each state.

\begin{restatable}{lemma}{lemmaTreeOpt} \label{lem:lb-opt}
    For every state $x$, it holds that $\opt(x) = \Delta \cdot \dist(x, z)$, where $\dist(x, z)$ denotes the length of the shortest path between $x$ and $z$ in the tree.
\end{restatable}
\vspace{-2.5mm}

The observations (with arbitrary noise) are set as follows.
Let $\P = (x_0, x_1, \ldots, x_{10d})$ be the path that starts at the root $x_0$ and ends at the target state $x_{10d}= z$.
We divide the states into three groups:
(1) those within a distance of $\epsilon d$ from the root;
(2) those in the subtree of $x_{\epsilon d}$; and
(3) everything else, i.e., states with a depth strictly larger than $\epsilon d$, that are not in the subtree of $x_{\epsilon d}$.
For a state $s$ in the first group, we let the observation $\oapx(x)$ be equal to $\Delta \cdot (10d - d_s)$, where $d_s$ is the depth of $s$ (its distance from the root).
That is, the true hitting time as if $s$ lied on the path from $x_0$ to $z$.
For the second group, we let $\oapx(s) := \Delta \cdot \dist(s, z)$, i.e., the observation is set without error to the optimal hitting time.
Finally, for the third group, we let $\oapx(s) := 10d - 2\epsilon d + d_s$.
That is, set to the hitting time as if $s$ were in the subtree of $x_{\epsilon d}$ but not in the subtree of $x_{\epsilon d + 1}$ (i.e., a node with depth $d_s$ that branches off $\P$ at $x_{\epsilon d}$).
It can be easily confirmed that these observations are within a $(1 \pm \epsilon)$-multiplicative error of the optimal hitting time (as characterized in \cref{lem:lb-opt}).

To prove \cref{thm:lb}, in \cref{sec:proofs}, we show that any algorithm that reaches the target state with constant probability using only the observations, must explore a constant fraction of the states in the first group (the vertices within distance $\epsilon d$ of $x_0$).
That is, it must take at least about $\Delta^{\epsilon d}$ steps which is exponentially larger than the optimal hitting time $10 \Delta d$.
\section{Experiments}
\label{sec:experiments}

\vspace{-2mm}
In this section, we evaluate a practical variant of our optimal rewinding algorithm (\cref{alg:opt}) that is better suited for practical settings, where the exact number of optimal steps from a state is not available. 

Our experiments and all the compared methods use the same underlying language model (GPT-4), the same prompting strategy for child generation, and comparable evaluator prompts.
We used ToT as the base code, and implemented CaT. The implementations are available on \href{https://github.com/almaghafari/caterpillar-of-thoughts}{github}.

\subsection{The Algorithm}
\vspace{-2mm}

For our experiments, we implement a variant of the optimal algorithm \cref{alg:opt} that instead of selecting the ``best'' explored state deterministically, selects one via the {\em softmax} distribution. This modification is more suitable for our empirical studies, as the optimal hitting time $\opt(x)$ from a given state $x$ is not readily available and can only be approximated. More formally, the distribution is based on $\oapx(x)$, an estimator of $\opt(x)$, and prioritizes states with smaller estimated hitting times. See \Cref{alg:softmax-rewinding}.

\begin{algorithm}[h]
\caption{Caterpillar of Thoughts (via softmax)}
\label{alg:softmax-rewinding}
 \begin{algorithmic}
    \STATE {\bfseries Input:} Initial state $x_0$, target state $z$, and hitting time estimator $\oapx(\cdot)$

    \STATE $S \gets \{x_0\}$

\WHILE{$z \notin S$}

    \STATE Sample a parent state $x^* \in S$ according to the softmax distribution
    \[
    \Pr[x^*] \propto \exp(- \oapx(x^*))
    \]

    \STATE Sample a child state $g \sim P(x^*,\cdot)$

    \STATE $S \gets S \cup \{g\}$.
    
\ENDWHILE
 \end{algorithmic}
\end{algorithm}

In each iteration, \Cref{alg:softmax-rewinding} rewinds to a previously observed state and expands it by sampling new child states according to the underlying transition probabilities.
The use of a softmax distribution heavily biases exploration toward states that are estimated to be closer to target, while still allowing exploration over the entire observed set.

We note that the main advantage of \cref{alg:softmax-rewinding} compared to the non-branching algorithm of \cref{thm:optimal-non-branching} is robustness to estimation error.
Without exact access to $\opt$, a non-branching algorithm may commit early to an incorrectly evaluated state $x$ and get stuck exploring a suboptimal subtree rooted at $x$, with no mechanism to recover from this mistake.
In contrast, \cref{alg:softmax-rewinding} allows the set of observed states to occasionally branch, preventing the algorithm from being confined to a single trajectory.
At the same time, the rewinding rule ensures that exploration remains exponentially biased toward states with smaller estimated rewinding hitting times. 

\subsection{Task: Game of 24}

In the \emph{Game of 24}, the model is given four integers and must construct a valid arithmetic expression using the operations $\{+,-,\times,/\}$ that evaluates exactly to $24$. Each number must be used exactly once, and intermediate results must be valid real numbers.
A solution is naturally constructed incrementally by applying operations step by step, which makes the task suited for modeling reasoning as a search process over partial solutions. For example, given input "4 4 6 8", a solution output
could be "(6 - 4) * (4 + 8) = 24".

\paragraph{The Markov Chain:}
A state corresponds to a \emph{partial solution}, consisting of
 the sequence of arithmetic operations applied so far, and
     the resulting multiset of remaining values.
From a given state, transitions correspond to sampling a valid next arithmetic operation and two of the remaining values using the language model, producing a new partial solution.
For example given input ``4 4 6 8", a valid state is represented as 
``4 + 8 = 12 (left: 4 6 12)"
or
``4 + 8 = 12 (left: 4 6 12), 
6 - 4 = 2 (left: 2 12)".
Any state representing a complete arithmetic expression that evaluates to $24$ and uses the input numbers each exactly once, is considered a target state.
\begin{figure}[t]
    \centering
    \includegraphics[width=0.65\linewidth]{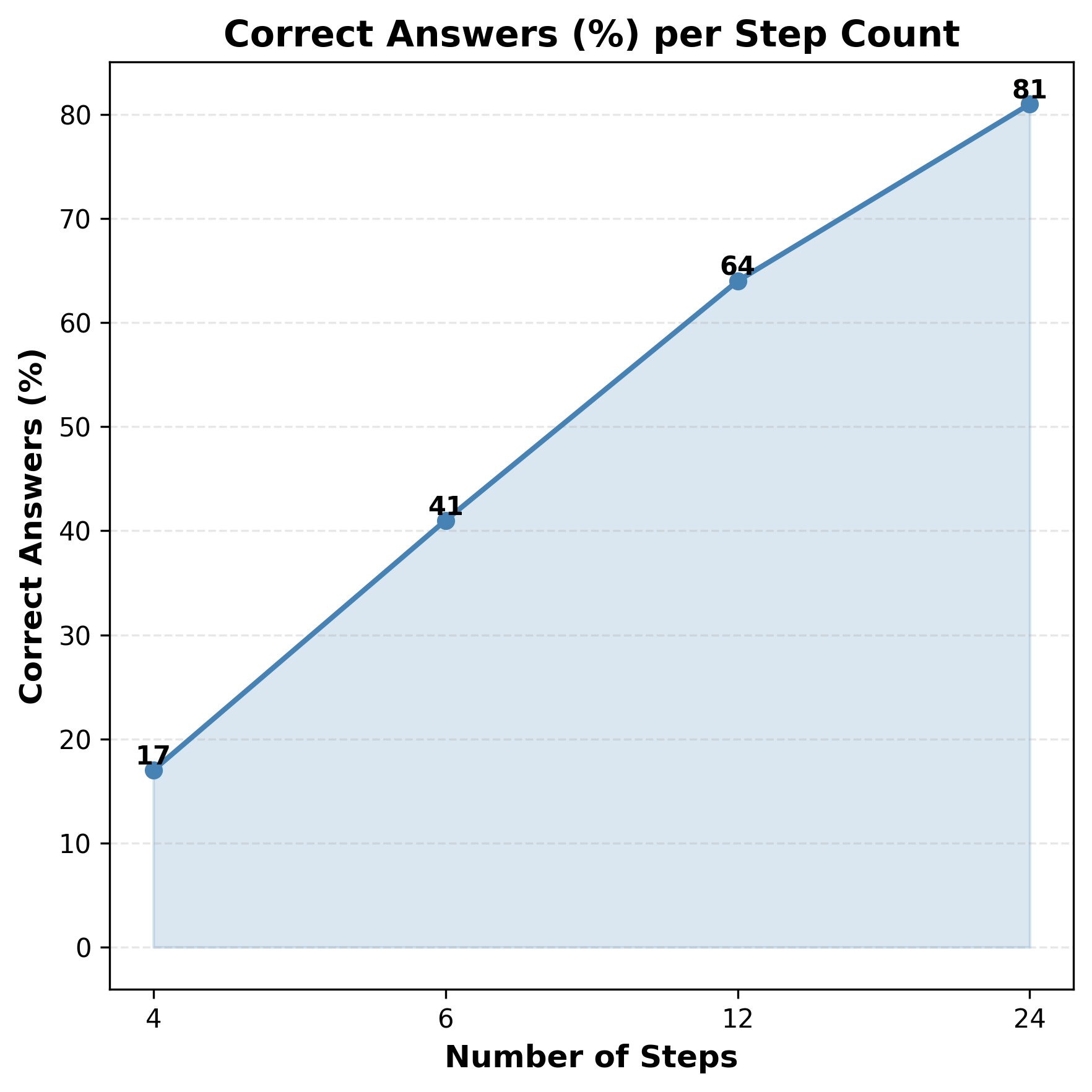}
    \caption{Illustration shows log linear relation between number of valid expressions and number of steps.
    }
    \label{fig:steppercor}
\end{figure}

\textbf{Data:} We use data from 4nums\footnote{\href{https://www.4nums.com/game/difficulties/}{4nums.com}}, the same data as \cite{yao2023treethoughtsdeliberateproblem} which has 1,362 games that are sorted from easy to hard, and use a subset of relatively hard games indexed 901-1,000 for testing. We report the success rate across 100 games as the metric.
\begin{table*}[th!]
\begin{minipage}{1\textwidth}
\caption{Performance on the $5\times 5$ crossword benchmark: cost matched ToT vs.\ CaT.}
\label{tab:crossword}

\centering
\begin{tabular}{lcccc}
\toprule
Method & Word Acc. (\%) & Letter Acc. (\%) & Games Solved (\%) & Avg. Tokens \\
\midrule
ToT \cite{yao2023treethoughtsdeliberateproblem}  & 39.5 & 64.8 & 5  &
73.4k \\
\midrule
\textbf{CaT} (\textbf{This paper}) & \textbf{50.0} & \textbf{68.6} & \textbf{15} & \textbf{66.8k} \\
\bottomrule
\end{tabular}
\end{minipage}
\end{table*}
\textbf{Algorithm Implementation}
In our experimental implementation, we use a  version of \cref{alg:softmax-rewinding} with bounded number of iterations.
 We fix a maximum number of iterations as a hyperparameter,
we execute the algorithm until a target state is reached, or terminate the algorithm once this budget of maximum iterations is exhausted.
The output is considered successful if a valid solution state is discovered within this budget.
In contrast, the BFS like algorithm of ToT, cannot terminate when reaching the target, since ToT prioritizes breadth instead of depth, and to hit the target needs to reach a specific depth and exhaustively process each layer.

In \cref{alg:softmax-rewinding}, at each iteration, after sampling a parent state according to the softmax distribution with temperature 1, we expand the selected state by sampling multiple child states.
The number of children sampled from a parent is not fixed a priori and is instead determined by the language model, following the same protocol used in Tree of Thoughts.
This allows the model to adaptively propose multiple candidate next steps from a given partial solution.

For each observed state $s$, we estimate $\opt(s)$ using the language model itself.
Concretely, the model is prompted to assess how likely the current partial solution is to lead to a correct final answer. The resulting scalar score is treated as a noisy estimate $\oapx(s)$ of the optimal remaining cost to go from $s$.

For the proposed algorithm, we run two independent executions (best of $2$), each with a fixed iteration budget.
We report success if either of the two runs produces a valid solution.
This design choice accounts for the stochastic nature of the rewinding process.
While \cref{alg:softmax-rewinding} significantly reduces the risk of committing to an early incorrect state, the algorithm may still, concentrate exploration on an unproductive region of the state space.
Running multiple independent instances avoids this effect without increasing the depth or branching factor within any single run.
In contrast, Tree of Thoughts performs a more exhaustive and structured tree expansion within a single execution.
Using best of $2$ allows for a fairer comparison between the two algorithms, while keeping average token used less than Tree-of-Thoughts.

\textbf{Results:}
\Cref{tab:game24-results} summarizes performance on the Game of 24 benchmark.
\cref{alg:softmax-rewinding} gets a 81\% accuracy in Game of 24 while reducing token usage to $15.3k$ tokens.
In \cref{fig:steppercor}, we illustrate how many of the inputs were solved per new iterations.

\vspace{-2mm}
\subsection{Task: Crossword}
We evaluate our method on a $5\times 5$ crossword completion benchmark.
Each instance consists of $10$ clues whose answers must be placed on a shared $5\times 5$ grid, with $5$ entries oriented horizontally and $5$ entries oriented vertically.
We report three metrics: (i) \emph{letter accuracy}, measured over the $25$ grid cells; (ii) \emph{word accuracy}, measured over the $10$ entries; and (iii) \emph{game success}, the fraction of instances for which the entire crossword is solved.

\textbf{The Markov Chain:}
We model the search process as a Markov chain that allows rewinding over partially completed boards.
A state is a partially filled crossword grid (equivalently, a set of assigned clue answers that are mutually consistent).
A transition selects one clue and proposes an answer using the language model; if the proposed entry is compatible with all currently placed letters at intersecting cells, it is inserted into the grid, producing a new (more complete) state. Incompatible proposals are rejected (the state remains unchanged).
 
\textbf{Data:}
We use the $5\times 5$ crossword dataset from GooBix,\footnote{\href{https://www.goobix.com/crosswords/0505/}{https://www.goobix.com/crosswords/0505/}}
which is the same source used in the Tree-of-Thoughts (ToT) experiments.
The dataset contains $156$ crossword instances.
Following the ToT setup, we evaluate on $20$ held-out puzzles with indices $1,6,\dots,96$, and we use the same subset as ToT for in-context prompting.

\textbf{Implementation:}
For evaluation, the model is prompted to assess how likely an answer to a clue will be compatible with the given the partially completed board. We score each partially filled crossword by adding the scores of each filled word in that table. The resulting scalar score is treated as a noisy estimate $\oapx(s)$ of the optimal remaining cost to go from $s$.
 Using these scores we generate the corresponding softmax distribution, and sample partially completed boards. Given the sampled board, we generate new boards by sorting the scores for each remaining hint, and picking top three. Note that for scoring hints we use the same prompts used in ToT.
In our implementation for this task, per input we spend 20 iterations of \cref{alg:softmax-rewinding}.

\textbf{Results:}
\Cref{tab:crossword} compares our approach (CaT) against a truncated version \footnote{This is from reports of \cite{chi-etal-2025-thoughtsculpt} on truncated ToT} of ToT with 20 search steps instead of 100 steps reported by ToT.
Across all metrics, CaT achieves higher accuracy and solves more puzzles while using fewer tokens on average.

\section*{Acknowledgements}
The third author thanks Shayan Talaei for many helpful discussions regarding the experiments.

\bibliography{ref}
\bibliographystyle{icml2026}

\newpage
\appendix
\onecolumn
\section{Comparison to Related Work}\label{sec:vgb}
In this section, we argue that Verifier Guided Backtracking (VGB)   \citet{rohatgi2025tamingimperfectprocessverifiers} is a special case of the Markov chain with rewinding.
We restate the VGB in a self-contained way.
Fix:
(1) a prompt $x$;
(2) a discrete token alphabet $\mathcal{A}$;
(3) a reference policy (base model) $\pi_{\mathrm{ref}}(\cdot \mid x, y_{1:h})$, which for every prefix
$y_{1:h}\in\mathcal{A}^h$ is a probability distribution over next tokens $a\in\mathcal{A}$;
and
(4) a nonnegative base value function $V_b(x,y_{1:h})\ge 0$ defined for every prefix $y_{1:h}\in\mathcal{A}^h$
(and similarly $V_b(x,y_{1:h},a)\ge 0$ defined for extended prefixes). 

Define the \emph{prefix tree}
$
\mathcal{T} \;:=\; \bigcup_{h=0}^{H} \mathcal{A}^h,
$
whose root is the empty prefix $\emptyset\in\mathcal{A}^0$.
For $y_{1:h}\in\mathcal{A}^h$ with $h\ge 1$, its parent is $\mathrm{par}(y_{1:h}) := y_{1:h-1}$.
For $h\le H-1$ where $H$ is the maximum depth of the tree, its children are $\mathrm{ch}(y_{1:h}) := \{(y_{1:h},a): a\in\mathcal{A}\}$.

Define the \emph{neighborhood} of a node $y\in\mathcal{T}$ by
$
N(y) := \{\mathrm{par}(y)\}   \cup  \mathrm{ch}(y) \}.
$
Given the current node $y=y_{1:h}$, define non-normalized weights on $N(y)$ by
\[
w(u \mid x,y) \;:=\;
\begin{cases}
V_b(x,y) & \text{if } u=\mathrm{par}(y)\ \text{and}\ y\neq\emptyset,\\[4pt]
\pi_{\mathrm{ref}}(a \mid x,y)\cdot V_b(x,y,a)
& \text{if } u=(y,a)\in\mathrm{ch}(y),\\[4pt]
0 & \text{otherwise.}
\end{cases}
\]
Let
$
p(u\mid x,y) := \frac{w(u\mid x,y)}{ \sum_{u\in N(y)} w(u\mid x,y)} \ \ \text{for } u\in N(y).
$

The VGB algorithm operates by evolving a Markov chain $(Y_t)_{t\ge 0}$.
The chain is initialized at $Y_0=\emptyset$ and evolves as follows:
given $Y_t=y$, sample $Y_{t+1}\sim p(\cdot\mid x,y).$
The process continues until a target state is reached.

\begin{lemma}
\label{lem:vgb}
The VGB process is an instance of the Markov chain with rewinding.
\end{lemma}

\begin{proof}
Instantiate our framework with the nodes of the prefix tree $\mathcal{T}$ as the state space.
The initial state is the root of the prefix tree corresponding to the prompt $x$, and an empty string $y = \emptyset$ of tokens generated so far.
For a node in the tree, the transition probabilities are defined on the children \emph{only}, proportional to the non-normalized weights.

With this setup, the VGB algorithm can be interpreted in our model as follows.
Given the most recently produced state $u_t = (x, y)$ (corresponding to the prompt $x$, and the string of generated tokens $y$),
the algorithm either rewinds to the parent $\mathrm{par}(u_t)$ with probability $p(\mathrm{par}(y) \mid x, y)$, or evolves a child otherwise.
As a result, the next state $u_{t+1}$ has the desired distribution as in VGB.
\end{proof}

\section{Omitted Proofs}
\label{sec:proofs}

This section includes the proofs deferred from the main body of the paper. The claims are restated here for ease of reference.

\lemmaOptByS*

To prove the lemma, we first show that the optimal algorithm depends only on the set of distinct elements of a sequence $X$ and not on the order or multiplicity of states in the history.

\begin{lemma}\label{clm:opt_s}
    Let  $X^{(1)}$ and $X^{(2)}$ be two different sequences of obtained states with the same set of distinct states $S$. We have 
    $$\opt(X^{(1)}) = \opt(X^{(2)}).$$
\end{lemma}

\begin{proof}
Fix an arbitrary algorithm $\sigma$ used after history $X^{(1)}$.
We construct a corresponding algorithm $\sigma'$ to be used after $X^{(2)}$ such
that the trajectories under $\sigma$ and $\sigma'$ can be perfectly
coupled while maintaining the following invariant.
Let $S_t^{(1)}$ and $S_t^{(2)}$ denote the sets of states observed up to time $t$
under $\sigma$ and $\sigma'$, respectively.
We maintain the invariant
\[
S_t^{(1)} = S_t^{(2)} = S_t \qquad \text{for all } t \ge 0.
\]
This holds initially because $X^{(1)}$ and $X^{(2)}$ have the same set of distinct states $S$.

Now assume the invariant holds at step $t$.
algorithm $\sigma$ rewinds to some time $t' \le t$, equivalently, it chooses a state
$a \in S_t$.
We define $\sigma'$ so that it also rewinds to the same state $a$.
Since $a \in S_t^{(1)} = S_t^{(2)}$, this is always possible.
Because the transition probabilities depend only on the chosen state and not on the
history, both runs now sample
$y \sim P(a,\cdot).$
We couple them so that both runs obtain the same next state $y$.
Therefore,
$$
S_{t+1}^{(1)} = S_t^{(1)} \cup \{y\}
              = S_t^{(2)} \cup \{y\}
              = S_{t+1}^{(2)},
$$
so the invariant holds at time $t+1$.
Thus, under this coupling, the sets of distinct states for $X^{(1)}$ and $X^{(2)}$ evolve identically at any time $t$ and therefore, both algorithm $\sigma$ and $\sigma'$ hit $z$ at the same time.

Let $\sigma(X^{(1)})$ be the expected hitting time of algorithm $\sigma$ for $X^{(1)}$. We define $\sigma'(X^{(2)})$ respectively. Pick $\sigma$ be the algorithm with optimal expected hitting time on  $X^{(1)}$.
For the coupled algorithm $\sigma'$ on $X^{(2)}$, we have 
$$
\opt(X^{(2)}) \le \sigma'(X^{(2)}) = \sigma(X^{(1)}) = \opt(X^{(1)}).
$$
Symmetrically, we also obtain
$\opt(X^{(2)}) \ge \opt(X^{(1)}).$
Hence $\opt(X^{(1)}) = \opt(X^{(2)}).$
\end{proof}

\begin{proof}[Proof of \cref{lem:opt-by-S}]
   Let $X$ be an arbitrary sequence of obtained states, and let $S$ be its set
    of distinct states.
    Since $S$ is precisely the set of distinct states in $X$, we can construct a
     sequence $\widetilde{X}$ that visits exactly the states in $S$
    (each at most once) and has the same set of distinct states as $X$.
    We denote this sequence by $S$ as well, viewed
    now as some fixed sequence whose distinct-state set is $S$.
    Applying \cref{clm:opt_s} with $X^{(1)} = X$ and $X^{(2)} = S$, we obtain $ \opt(X) = \opt(S).$
    This shows that the optimal hitting time starting from any history $X$
    depends only on its set of distinct states $S$.
\end{proof}

\corOptRecursion*
\begin{proof}
    Consider the hitting time of the optimal algorithm described in \cref{thm:optimal-non-branching},
    and let $y$ be the first state drawn by the algorithm
    (obtained by rewinding to $x_0$).
    There are two possibilities for $y$.
    Either $\opt(y) \geq \opt(x)$, in which case it is ignored and the expected number of remaining steps is $\opt(x)$,
    or $\opt(y) < \opt(x)$, in which case $x^*$ is set equal to $y$ and the expected remaining number of steps is $\opt(y)$.
    Therefore, it holds that
    \begin{align*}
        \opt(x) &= 1 + \sum_{y \notin L}P(x, y) \opt(x) + \sum_{y \in L}P(x, y) \opt(y) \\
        &= 1 + P(x, \Omega\setminus L) \opt(x) + \sum_{y \in L}P(x, y) \opt(y).
    \end{align*}
    The claim follows from moving $P(x, \Omega\setminus L) \opt(x)$ to the left-hand side and dividing both sides by $P(x, L)$.
\end{proof}

\thmDijkstra*

\begin{proof}
    We prove the following invariant inductively:
    At any point, for all states $x \in \Omega$, it holds that $d_x \geq \opt(x)$, and this holds with equality for all states $x \in T$.
    For the base case, we consider the algorithm at the end of the first iteration.
    At this time, it holds that $T = \{z\}$ and $d_z = \opt(z) = 0$.
    For all $x \neq z$, we have
    $$
    d_x = \begin{cases}
        \frac{1}{P(x, z)} & \textnormal{if $P(x, z) \neq 0$, and}  \\
        \infty & \textnormal{otherwise.}
    \end{cases}
    $$
    This value is an upper bound for $\opt(x)$ as it corresponds to the simple algorithm which rewinds to $x$ until it obtains the target state $z$.

    For the induction step, given the set $T$, we define a corresponding algorithm for each initial state $x \notin T$.
    The algorithm starts by repeatedly rewinding to $x$, until it obtains a state $g \in T$.
    Then, it proceeds by following the optimal algorithm for $g$.
    The expected hitting time of this algorithm can be analyzed as follows.
    In expectation, it takes $\frac{1}{P(x, T)}$ to obtain a state in $T$ from $x$.
    Then, conditioned on that the newly obtained state $g$ is in $T$, we have $g = y$ with probability $\frac{P(x, y)}{P(x, T)}$,
    in which case the remainder of the algorithm takes $\opt(y)$ steps to terminate in expectation.
    Therefore, the expected hitting time of this algorithm is equal to:
    $$
    \frac{1}{P(x, T)} + \sum_{y \in T}\frac{P(x, y)}{P(x, T)}\opt(y).
    $$
    Observe that if $\opt(y) = d_y$ for all $y \in T$, then the value above is equal to $d_x$.
    
    Hence, it remains to show that when $x^*$ is added to $T$, it holds that $d_{x^*} = \opt(x^*)$.
    Let $A \subseteq \Omega \setminus T$ be the set of states $x$ outside $T$ that minimize $\opt(x)$,
    and let $m$ denote the minimum.

    For any $x \in A$, the algorithm corresponding to $d_{x}$ (as defined in the previous paragraph) is the same as the optimal algorithm outlined in \cref{thm:optimal-non-branching} since $T$ contains every state $x'$ with $\opt(x') < \opt(x)$.
    Therefore, we have that $d_{x} = \opt(x) = m$ for all $x \in A$.
    Furthermore, for any $x \notin A$, it holds that $d_x \geq \opt(x) > m$.
    As a result, $x^* := \argmin_{x \notin T} d_x$ is a member of $A$, and it holds that $d_{x^*} = \opt(x^*) = m$.

    Finally, we note that when the algorithm terminates, for any state $x \notin T$, there is no path of positive probability from $x$ to $z$. Otherwise, the algorithm would have expanded $T$ through one of the edges in the path that go from $\Omega \setminus T$ to $T$.
    As a result, for any $x \notin T$, it holds that $\opt(x) = d_x = \infty$.
    Therefore, when \cref{alg:dijkstra} terminates, $d_x$ is equal to $\opt(x)$ for all $x \in \Omega$.
    It can be easily seen that it can be implemented in $O(n^2)$ or $O( (n + m) \log n)$, which concludes the proof.
\end{proof}

\begin{lemma} \label{clm:aux-recursion}
    For a state $x \neq z$, let $L \subseteq \Omega$ be the set of states $y$ such that $\opt(y) < \opt(x) - c$.
    Then, it holds that
    $$
    \opt'(x) = \frac{1 + \sum_{y \in L}P(x, y) \opt'(y)}{P(x, L)}.
    $$
\end{lemma}
\begin{proof}
    The proof is similar to that of \cref{cor:opt-recursion}.
    After the first state $y$ is drawn,
    either $y \notin L$, in which case the expected remaining number of steps is $\opt'(x)$,
    or $y \in L$, in which case the expected remaining number of steps is $\opt'(y)$.
\end{proof}

\lemmaAuxAnalysis*

\begin{proof}
    Given that $c = \frac{\epsilon}{1 + \epsilon}$ holds, $\opt'(x) \leq (1 + \epsilon)\opt(x)$ and $(1 - c) \opt'(x) \leq \opt(x)$ are equivalent.
    We prove the latter inductively, ordering the states by $\opt(x)$.
    More specifically, for the base case, we consider $x = z$ where $\opt(x) = \opt'(x) = 0$.
    Then, for the induction step, we assume the claim for all $y$ with $\opt(y) < \opt(x)$, and prove it for $x$.

    Take any state $x \neq z$. We divide the states into three groups based on their optimal hitting time:
    \begin{align*}
    H &:= \{y \in \Omega \mid \opt(x) \leq \opt(y) \}, \\
    M &:= \{y \in \Omega \mid \opt(x) - c \leq \opt(y)  < \opt(x) \},\ \textnormal{and} \\
    L &:= \{y \in \Omega \mid \opt(y) < \opt(x) - c \}.
    \end{align*}
    Recall the recursive formulation of $\opt(x)$ (\cref{cor:opt-recursion}, note the change in definition of $L$):
    $$
    \opt(x) = \frac{1 + \sum_{y\in M \cup L}P(x, y) \opt(y)}{P(x, M \cup L)}.
    $$
    For $y \in M$, it holds $\opt(y) \geq \opt(x) - c$ by definition.
    For $y \in L$, the induction hypothesis implies $\opt(y) \geq (1 - c)\opt'(y)$.
    Therefore, $\opt(x)$ can be lower-bounded as follows:
    \begin{align*}
        \opt(x) &\geq \frac{1 + \sum_{y\in L}P(x, y) (1-c)\opt'(y)}{P(x, M \cup L)} \\ 
        & + \frac{ \sum_{y \in M}P(x, y)(\opt(x) - c)}{P(x, M \cup L)}\\
        & = \frac{1 + \sum_{y\in L}P(x, y) (1-c)\opt'(y)}{P(x, M \cup L)} \\
        & + \frac{P(x, M)}{P(x, M \cup L)}(\opt(x) - c).
    \end{align*}
    Moving $\frac{P(x, M)}{P(x, M \cup L)}\opt(x)$ to the left-hand side and dividing both sides by $1 -  \frac{P(x, M)}{P(x, M \cup L)}$, we get
    $$
    \opt(x) \geq \frac{1 + \sum_{y\in L}P(x, y) (1-c)\opt'(y)}{P(x, L)} - \frac{P(x, M)}{P(x, L)} c.
    $$
    Further algebraic manipulation of the inequality and plugging in the recursive formulation of $\opt'(x)$ (\cref{clm:aux-recursion}) yields
    \begin{align*}
        \opt(x)
        &\geq \frac{1 + \sum_{y\in L}P(x, y) (1-c)\opt'(y)}{P(x, L)} - \frac{P(x, M)}{P(x, L)} c \\
        &= (1-c)\frac{1 + \sum_{y\in L}P(x, y) \opt'(y)}{P(x, L)}  \\ 
        & + \frac{1}{p(x, L)}c - \frac{P(x, M)}{P(x, L)} c \\
        &\geq (1-c)\frac{1 + \sum_{y\in L}P(x, y) \opt'(y)}{P(x, L)} \\
        &= (1 - c)\opt'(x).
    \end{align*}
    This concludes the proof.
\end{proof}

\meanMedian*

\begin{proof}
    Let $N = \log \frac{1}{\delta}$.
    First, for each $1 \leq i \leq N$, let $X^{(i)}$ be the mean of $k = 32 \frac{\lambda^2}{\epsilon^2}$ independent approximations obtained from $\oapx(s)$.
    As a result, each $X^{(i)}$ has a mean of $\opt(s)$ and a variance of $\frac{2\lambda^2}{k} = \frac{\epsilon^2}{16}$.
    Therefore, by Chebyshev's inequality, it holds:
    $$
    \Pr[\card{X^{(i)} - \opt(s)} \geq \epsilon] \leq \frac{\epsilon^2 / 16}{\epsilon^2} \leq 
    \frac{1}{16}.
    $$

    We let the final estimate $X_s$ be equal to the median $\{X^{(i)}\}_{i \leq N}$, and prove that it satisfies
    $$
    \Pr[\card{X_s - \opt(s)} > \epsilon] \leq \delta.
    $$
    To show this, we note that for the median to be outside the interval $\opt(s) \pm \epsilon$, at least half of the means $\{X^{(i)}\}_{i \leq N}$ must be outside the interval.
    Taking the union bound over all the possibilities for these $\floor{N / 2}$ indices, the probability of this event is at most
    \begin{align*}
        \binom{N}{\floor{N/2}} \cdot \left(\frac{1}{16}\right)^{N/2}
        &\leq 2^N \cdot \left(\frac{1}{16}\right)^{N/2} \\
        &\leq 2^{-N} \\
        &\leq \delta.
    \end{align*}
    Therefore, $X_s$ has the desired property, which concludes the proof.
\end{proof}

\thmStable*

\begin{proof}
    Without loss of generality, we assume that the guarantee of the mean-median approximation (\cref{lem:mean-median}) holds in every step of \cref{alg:stable} deterministically, and show that the expected number of steps is at most $2\opt(x_0)$ when it does.
    That is,
    \begin{equation}
    \card{X_{x^*} - \opt(x^*)} \leq \frac{1}{10} \qquad \text{and} \qquad \card{X_{g} - \opt(g)} \leq \frac{1}{10}. \label{eq:mm-guarantee}
    \end{equation}
    We show this is without loss of generality through a slight modification to the algorithm.
    We add a reset mechanism to the algorithm:
    If the target is not reached within $4N$ steps, then the algorithm discards the run and starts from the beginning.
    Each run involves $8N$ mean-median approximations (two per each step),
    and by \cref{lem:mean-median} each of them satisfies the guarantee with probability $\delta = \frac{1}{10 N}$.
    Therefore, the probability that every mean-median approximation in a run satisfies the guarantee is at least:
    $$
    \left(1 - \frac{1}{10 N}\right)^{8N} \geq \frac{e^{-8/10}}{2} \geq \frac{1}{5}.
    $$
    When this event happens, we show that the expected number of steps before reaching the target state is at most $2N$.
    As a result, by the Markov inequality, such a run has a probability of at least $\frac{1}{2}$ to reach the target before the reset at $4N$ steps.
    That is, in expectation, it takes $\frac{1}{1/5} \cdot \frac{1}{1/2} = 10$ runs to reach the target state, incurring only a constant factor on the number of steps.

    Hence, it remains to show that the expected number of steps to reach the target state is at most $2 \opt(x_0)$, assuming \eqref{eq:mm-guarantee} holds in every step.
    The analysis is through comparing the stable algorithm (\cref{alg:stable}) to the auxiliary algorithm (\cref{def:aux}).
    Invoking \cref{lem:aux-analysis}, with $\epsilon = 1$ and $c = \frac{1}{2}$, yields
    $$
    \opt'(x_0) \leq 2 \opt(x_0).
    $$
    We show that the stable algorithm has at most as many steps as the auxiliary rewinding algorithm.
    More precisely, let $\opt''(x)$ denote the expected number of steps of the stable algorithm, assuming \eqref{eq:mm-guarantee} holds.
    We prove that $\opt''(x) \leq \opt(x)$ for all $x$.

    We prove this claim using induction.
    Ordering the states increasingly based on $\opt(x)$, 
    we assume $\opt''(y) \leq \opt'(y)$ for all $y$ before $x$ (i.e., states $y$ with $\opt(y) < \opt(x)$), and show the claim holds for $x$.
    It holds trivially for the base case: $\opt''(z) = \opt'(z) = 0$.
    For the induction step, recall the recursive formulation of $\opt'(x)$:
    \begin{equation}
        \opt'(x) = \frac{1 + \sum_{y \in L}P(x, y) \opt'(y)}{P(x, L)} = 1 + \sum_{y \notin L} P(x, y) \opt'(x) + \sum_{y \in L}P(x, y) \opt'(y), \label{eq:rec1}
    \end{equation}
    where $L \subseteq \Omega$ is the set of states $y$ such that $\opt(y) < \opt(x) - c$.
    
    Similarly, we can obtain a recursive formulation for $\opt''(x)$.
    Let $M \subseteq \Omega$ be the set of states $y$ with $\opt(x) - c < \opt(y) < \opt(x)$.
    We note that since \eqref{eq:mm-guarantee} holds, the stable algorithm replaces $x^*$ by $g$ when $g \in L$, and it does not replace $x^*$ by $g$ when $g \notin M \cup L$.
    The only difference between the stable algorithm and the auxiliary algorithm is that the stable algorithm might replace $x^*$ for some of the states $g \in M$, with a probability of $Q(x^*, g)$, in which case it gains a benefit over the auxiliary algorithm.
    Therefore, we obtain
    $$
        \opt''(x) = 1 + \sum_{y \notin M \cup L} P(x, y) \opt''(x) 
        + \sum_{y \in M} P(x, y) (Q(x, y)\opt''(y) + (1 - Q(x, y))\opt''(x))
        + \sum_{y \in L}P(x, y) \opt''(y).
    $$
    Using the induction hypothesis we replace $\opt''(y)$ with $\opt'(y)$ to obtain:
    $$
        \opt''(x) \leq 1 + \sum_{y \notin M \cup L} P(x, y) \opt''(x) 
        + \sum_{y \in M} P(x, y) (Q(x, y)\opt'(y) + (1 - Q(x, y))\opt''(x))
        + \sum_{y \in L}P(x, y) \opt'(y).
    $$
    Rewriting \eqref{eq:rec1}, we have:
    $$
    \opt'(x) = 1 + \sum_{y \notin M \cup L} P(x, y) \opt'(x) 
        + \sum_{y \in M} P(x, y) (Q(x, y)\opt'(x) + (1 - Q(x, y))\opt'(x))
        + \sum_{y \in L}P(x, y) \opt'(y).
    $$
    Combining the two implies $\opt''(x) \leq \opt'(x)$, which concludes the proof of the induction.
    This yields the desired bound of $\opt''(x) \leq 2 \opt(x)$ on the number of steps.
    Furthermore, each step uses $O(\frac{\lambda^2}{\epsilon^2}\log\frac{1}{\delta}) = O(\lambda^2 \log \opt(x_0))$ Laplacian approximations.
    Therefore, the total cost is $O(\lambda^2 \cdot \opt(x_0) \log \opt(x_0))$.
    In particular, when $\lambda = \epsilon \opt(x_0)$, the total cost is $O(\epsilon^2 \cdot \opt(x_0)^3 \log \opt(x_0))$. This completes the proof.

    We also remark that when an upper bound $N$ is not available for $\opt(x_0)$, a scaled-up approximation of $2 \oapx(x_0)$ can be used for each run.
\end{proof}

\lemmaTreeOpt*

\begin{proof}
    The optimal hitting time is characterized by \cref{thm:optimal-non-branching}.
    Take any non-leaf state $x$, and let $N(x)$ be its neighbors in the tree.
    Since the underlying graph is a tree, there exists a neighbor $y^* \in N(x)$  such that $\dist(x, z) = \dist(y, z) + 1$, and for any other neighbor $y \neq y^*$ it holds that $\dist(x, z) = \dist(y, z) - 1$.
    Note that any path starting $x$ or any other neighbor in $N(x) \setminus \{y^*\}$ and ending in $z$, must go through $y^*$.
    Therefore, it holds $\opt(x) \geq \opt(y^*)$, and $\opt(y) \geq \opt(y^*)$ for any neighbor $y \neq y^*$.

    Based on this, we can invoke \cref{thm:optimal-non-branching}.
    There exists an optimal non-branching algorithm that operates as follows:
    \begin{enumerate}
        \item Maintain a \enquote{minimizer} state $x^*$, initially equal to $x_0$.
        \item Repeatedly rewind to $x^*$ to obtain a state $y$, until the new state satisfies $\dist(y, z) < \dist(x^*, z)$. \label{step:tree-alg-2}
        \item If $y$ is equal to the target state $z$, halt.
        Otherwise, let $x^* := y$ and go back to step \ref{step:tree-alg-2}.
    \end{enumerate}
    Equivalently, the algorithm moves on the shortest path from $x_0$ to $z$.
    At each vertex $x$, it repeatedly generates neighbors until the next vertex on the shortest path is drawn.
    Since the next vertex on the shortest path is drawn with probability $1/\Delta$, this takes $\Delta$ steps in expectation, and the entire process lasts $\Delta \cdot \dist(x_0, z)$ in expectation.
\end{proof}

\thmLB*
\begin{proof}
    We let $\Delta = 2^{1/\epsilon}$, so that $\opt(x_0) = 2^{1/\epsilon} 10 d$, and $d$ can be set arbitrarily large.
    By Yao's minimax lemma, it suffices to present an input distribution, such that any \emph{deterministic} algorithm that reaches the target state with at least constant probability uses at least $\Omega(2^{\opt(x_0)})$ queries.
    In fact, we show that this many queries is necessary for reaching $x_{\epsilon d}$ with constant probability, where recall that $x_0, \ldots, x_{10d}$ is the path starting at the initial state $x_0$ and ending at the target state $x_{10d} = z$.

    We construct the input distribution based on the Markov chain outlined in \cref{sec:adversarial-noise}.
    The structure of the tree and the node labels are fixed, then one of the leaves is chosen uniformly at random to be the target state $z$, and the observations are set accordingly.
    The observations in the shallow part of the tree (i.e., among the vertices within distance $\epsilon d$ of the root) the observations $\oapx(\cdot)$ are symmetric, in the sense that any two states of the same depth have the same observation.
    Furthermore, considering any two nodes at depth $\epsilon d$, the observations in the two subtrees are identical, except for the node $x_{\epsilon d}$.
    We refer to the nodes at depth $\epsilon d$ as the middle nodes, and call $x_{\epsilon d}$ the special middle node.
    As such, the observations made within depth $\epsilon d$ reveal nothing about the location of the $x_{\epsilon d}$,
    and any observations made in the subtree of a middle node $s$ only reveals whether $s$ is special or not,
    and reveals no further information about the location of $x_{\epsilon d}$ relative to $s$.
 
    At any point in the run time of the algorithm, we say that a middle node has been revealed if any observation is made in its subtree.
    By the previous paragraph, conditioned on the observations made so far, if none of the revealed vertices is special (i.e., if $x_{\epsilon d}$ has not been reached), then the special node could be any of the remaining middle nodes uniformly at random.
    As such, any algorithm that reveals a $c$-fraction of the middle nodes, reaches $x_{\epsilon d}$ with a probability of at most $c$.

    Therefore, any algorithm that reaches the target state $z$ (and hence the special middle node $x_{\epsilon d}$) with constant probability, must reveal at least $\Omega(\Delta^{\epsilon d}) = \Omega(\exp(\opt(x_0))$ middle nodes, and as a result, take at least as many steps. This concludes the proof.
\end{proof}

\end{document}